\documentclass{SCIS2021}
\begin{document}
%\oa
%%%%%%%%%%%%%%%%%%%%%%%%%%%%%%%%%%%%%%%%%%%%%%%%%%%%%%%
%%% Authors do not modify the information below
\ArticleType{RESEARCH PAPER}
%\SpecialTopic{}
%\luntan
\Year{2020}
\Month{}
\Vol{}
\No{}
\DOI{}
\ArtNo{}
\ReceiveDate{}
\ReviseDate{}
\AcceptDate{}
\OnlineDate{}
%%%%%%%%%%%%%%%%%%%%%%%%%%%%%%%%%%%%%%%%%%%%%%%%%%%%%%%
	
%%% title
%%%   \title{title}{title for citation}
\title{On the Robustness of Median Sampling\\in Noisy Evolutionary Optimization}{On the Robustness of Median Sampling in Noisy Evolutionary Optimization}

%%% Corresponding author
%%%   \author[number]{Full name}{{email@xxx.com}}
%%% General author
%%%   \author[number]{Full name}{}
\author[1]{Chao BIAN}{}
\author[1]{Chao QIAN}{{qianc@lamda.nju.edu.cn}}
\author[1]{Yang YU}{}
\author[2]{Ke TANG}{}

%%% Author information for page head
\AuthorMark{Bian C}

%%% Authors for citation
\AuthorCitation{Bian C, Qian C, Yu Y, et al}

%%% Authors' contribution
%\contributions{Authors A and B have the same contribution to this work.}

%%% Address
%%%   \address[number]{Affiliation, City {\rm Postcode}, Country}
\address[1]{State Key Laboratory for Novel Software Technology, Nanjing University, Nanjing {\rm 210023}, China}
\address[2]{Shenzhen Key Laboratory of Computational Intelligence, Department of Computer Science and Engineering, \\Southern University of Science and Technology, Shenzhen {\rm 518055}, China}

%%% Abstract 
\abstract{Evolutionary algorithms (EAs) are a sort of nature-inspired metaheuristics, which have wide applications in various practical optimization problems. In these problems, objective evaluations are usually inaccurate, because noise is almost inevitable in real world, and it is a crucial issue to weaken the negative effect caused by noise. Sampling is a popular strategy, which evaluates the objective a couple of times, and employs the mean of these evaluation results as an estimate of the objective value. In this work, we introduce a novel sampling method, median sampling, into EAs, and illustrate its properties and usefulness theoretically by solving OneMax, the problem of maximizing the number of 1s in a bit string. 
Instead of the mean, median sampling employs the median of the evaluation results as an estimate. Through rigorous theoretical analysis on OneMax under the commonly used onebit noise, we show that median sampling reduces the expected runtime exponentially. Next, through two special noise models, we show that when the 2-quantile of the noisy fitness increases with the true fitness,  median sampling can be better than mean sampling; otherwise, it may fail and mean sampling can be better. The results may guide us to employ median sampling properly in practical applications.}

\keywords{Evolutionary algorithms, noisy optimization, median sampling, computational complexity, runtime analysis}

\maketitle

%%%%%%%%%%%%%%%%%%%%%%%%%%%%%%%%%%%%%%%%%%%%%%%%%%%%%%%
%%% The main text
%%%%%%%%%%%%%%%%%%%%%%%%%%%%%%%%%%%%%%%%%%%%%%%%%%%%%%%
\section{Introduction}
As a kind of general-purpose optimization algorithms, evolutionary algorithms (EAs)~\cite{1} have wide applications in practical optimization problems~\cite{2,3}.
During the optimization procedure, the obtained objective (i.e., fitness) value is usually inaccurate because of noise~\cite{4}. For example, in machine learning, the estimated performance of a prediction model usually deviates from the true performance because the model is evaluated on a limited amount of data; in aerodynamic design, the computational fluid dynamics (CFDs) simulation is needed to evaluate the performance of a given structure, which is usually computationally expensive and approximated, leading to noisy fitness. The existence of noise may mislead the search direction and deteriorate the efficiency of EAs. Therefore, it is important to handle noise in fitness evaluation during evolutionary optimization.

The sampling strategy independently evaluates the fitness $m$ times, where $m$ is the sample size, and then the mean of these samples is used to estimate the exact fitness. Sampling is very popular to tackle noise, because it has a $m$-fold reduction in the variance of the noisy evaluation. Meanwhile, it also has a $m$-fold increase in the computation time, thus some variants are proposed, including adaptive sampling~\cite{5,6} and sequential sampling~\cite{7,8}, which decide the value of $m$ dynamically in each generation. However, there has been a great lack of the theoretical understanding for sampling.

Runtime analysis, an important theoretical aspect for EAs, has achieved a lot of progresses~\cite{9,10,11,12,13,14} recently. However, they mainly consider exact environments, and the results on noisy evolutionary optimization is rare. Noise increases the randomness in the optimization procedure, making the analysis more difficult. As a representative evolutionary algorithm, (1+1)-EA maintains one solution in the population, and generates a new solution in each iteration by mutating the parent solution. It was first studied on two frequently-used pseudo-Boolean problems, OneMax (OM) and LeadingOnes (LO). The goal of OM is maximizing the number of 1s in a solution, while the goal of LO is maximizing the number of continuous 1s from the first bit in a solution. 
Runtime analysis for the two problems under various noise models~\cite{15,16,17,18,19,20} showed that only if the noise level is low, (1+1)-EA can quickly find the optimum. For instance, onebit noise is a frequently-used noise model in theoretical analysis. With probability $p$, it changes a uniformly selected bit in a solution before evaluation, leading to a random fitness value. For OM of size $n$ under onebit noise, the expected runtime (ERT) of (1+1)-EA is superpolynomial if $p=\omega(\log n/n)$. There are also some studies concerning the effectiveness of various strategies to tackle noise, e.g., threshold selection~\cite{19,21,22}, populations~\cite{16,18,20,23,24} and  sampling~\cite{25,26,27}. For instance, if $\mu=\Theta(\log n)$, the ERT of $(\mu+1)$-EA optimizing OM under onebit noise is polynomial for any $p\in [0,1]$ (note that $p$ denotes the noise probability). Several works also show the robustness of the compact genetic algorithm~\cite{28} and a simple ant colony optimization algorithm \cite{29,30,31,32} against noise.

The above mentioned runtime analyses concerning sampling~\cite{25,26,27} revealed that the exponential runtime under high noise levels can be turned to be polynomial by sampling, and the sample size may be critical to the effectiveness of sampling. Moreover, Akimoto et al.~\cite{33} showed that 
optimization under unbiased noise can perform like exact optimization, if the sample size $m$ is large enough. In these works, the sampling strategy utilizes the mean of the samples as an approximation of the true fitness. Then a natural question is whether other information of the samples can be used to make EAs more robust against noise. 

Note that mean is actually a measure of central tendency, and thus, it is straightforward to use another widely known measure \textit{median}. Compared to mean, median has the advantage of being insensitive to outliers. 
For example, ``breakdown point"~\cite{34,35} is a commonly used indicator for insensitivity, which denotes the minimum ratio of variables that need to be contaminated to make the estimator become infinite (i.e., cause breakdown).
The breakdown point of mean is close to 0  because a single bad observation can make the mean become infinite, whereas the breakdown point of median is 0.5 because median becomes infinite only if more than 50\% of the variables become infinite. In fact, economists use the sample median frequently when reporting statistics concerning certain economic measures, e.g., household income~\cite{36}. 

In this paper, we introduce the sampling strategy using median (called \textit{median sampling}) into EAs and theoretically examine its effectiveness. Instead of taking the mean, median sampling takes the median of the samples as an estimate for the fitness. In order to better distinguish the two sampling strategies, we call the original sampling strategy \textit{mean sampling} in the following context. We will consider (1+1)-EA solving noisy OM, and derive the ERT for reaching the optimum (with respect to the exact objective).
Following is our main results:

\begin{itemize}
	\item  For OM under onebit noise with any $p\in [0,1]$, we prove that the ERT of (1+1)-EA is polynomial when median sampling with $m=2n^3+1$ is used. Previous analysis~\cite{17} has proved that the ERT of (1+1)-EA is polynomial only if $p=O(\log n/n)$. Thus, the result shows the robustness of median sampling against noise.
	\item  For OM under segmented noise, we show that the ERT of (1+1)-EA using median sampling is polynomial, while the ERT of (1+1)-EA using mean sampling is exponential.
	 The results show that median sampling can be a better choice, if the 2-quantile of the noisy fitness increases with the true fitness. Note that the noisy fitness is a random variable, and the 2-quantile of a random variable $X$ is the value $a$ satisfying $\mathrm{P}(X\le a)\ge 0.5$ and $\mathrm{P}(X\ge a)\ge 0.5$ (i.e., the median of $X$).
	\item For OM under partial noise, we show that (1+1)-EA employing median sampling fails, while (1+1)-EA employing mean sampling works. 
	The results suggest that it would be better to choose other strategies if the 2-quantile of the noisy fitness doesn't increase with the true fitness.
\end{itemize}
Note that in parallel with our work, Doerr and Sutton~\cite{37} showed that median sampling can handle the negative impact of noise for an integer valued objective function $f$, if $f^{\mathrm{n}}(x)$ satisfies the $\epsilon$-concentrate condition, that is, $\mathrm{P}(f^{\mathrm{n}}(x)-f(x)\ge 0.5)\le 0.5-\epsilon$ and $\mathrm{P}(f^{\mathrm{n}}(x)-f(x)\le -0.5)\le 0.5-\epsilon$, where $f^{\mathrm{n}}(x)$ denotes the noisy objective value of $x$. They also considered two specific cases to show the superiority of  median sampling over mean sampling. For OM under additive Cauchy noise with parameter $\gamma\ge 0.5$, they showed that the runtime of (1+1)-EA is superpolynomial w.h.p. (with high probability) if mean sampling is used, and the runtime is polynomial w.h.p. if median sampling is used. For LO under bitwise noise $(p,q)$~\cite{27} satisfying $p=0.5-\epsilon \wedge q=\Omega(1)$, they showed a  superpolynomial ERT for (1+1)-EA if mean sampling with $m=O(n/\log^2n)$ is used, and the runtime is polynomial w.h.p. if median sampling with $m=O(\log n)$ is used.

The remaining paper is presented as follows. First, Section~2 presents preliminaries. Then, Section~3 analyzes the effectiveness of median sampling. Next, Sections~4 and~5 compare median sampling with mean sampling, and Section~6 provides some guidance for employing median sampling in practice. Finally, Section~7 makes a conclusion.

\section{Preliminaries}
We first present the OM problem as well as (1+1)-EA which will be considered in this paper. Next, we present the sampling strategy. The analysis tool is presented in the end.

\subsection{OneMax Problem}
We consider the frequently-used pseudo-Boolean function OM. Its goal is maximizing the number of 1s (namely, the bits with value 1) in a solution. Note that 11...1 (denoted as $1^n$) is the unique optimal solution. The ERT of (1+1)-EA solving OM (without noise) is $\Theta(n\log n)$~\cite{38}. For notational convenience, $|x|_0$ will be used to represent the number of 0s (namely the bits with value 0) of $x$.
\begin{definition}\label{def_onemax}
	The goal of the OM Problem with size $n$ is finding a binary string $x^*$ to maximize $f(x)=\sum^{n}_{i=1} x_i$.
\end{definition}

\subsection{(1+1) Evolutionary Algorithm}
(1+1)-EA reflects the general structure of EAs, and is widely analyzed to theoretically understand the behavior of EAs. Different from exact optimization, only a noisy fitness value $f^{\mathrm{n}}(x)$ can be obtained in noisy environments, and the value is a random variable because the noise may disturb the solution or the objective value randomly. For example, there are two kinds of widely used noise models: posterior and prior. The posterior noise comes from the variation on the fitness of a solution, e.g., $f^{\mathrm{n}}(x)=f(x)+\delta$, where $\delta$ is randomly drawn from some distribution. The prior noise comes from the variation on a solution, i.e., $f^{\mathrm{n}}(x)=f(x')$, where $x'$ is generated from $x$ by random perturbations. Therefore, line~5 in Algorithm~\ref{(1+1)-EA} changes from the true fitness ``$f(\cdot)$" to the noisy fitness ``$f^{\mathrm{n}}(\cdot)$". In the optimization process, reevaluation strategy, which evaluates both the offspring and parent solutions in each generation, is used as in~\cite{17,18,29}.
In the optimization procedure of an EA, fitness evaluations are the most time-consuming part, thus we will simply define its runtime as the number of objective evaluations. Termination condition is the finding of the optimum w.r.t. the exact objective~\cite{17,18,33}. In this work, we consider maximizing $f: \{0,1\}^n\rightarrow \mathbb{R}$.
\begin{algorithm}[h]\caption{(1+1)-EA}\label{(1+1)-EA} 
    \begin{algorithmic}[1]
    \STATE  Let $x$ be a uniformly chosen solution in $\{0,1\}^n$.
    \STATE  While the stopping criterion doesn't satisfy
    \STATE  \quad $z\leftarrow$ copy $x$.
    \STATE  \quad Change each bit of $z$ with a probability of $1/n$ independently.
    \STATE \quad If {$f^{\mathrm{n}}(z) \geq f^{\mathrm{n}}(x)$} \quad then $x\leftarrow z$.
    \end{algorithmic}
\end{algorithm}

\subsection{Median Sampling}
Mean sampling has often been used in noisy evolutionary optimization to tackle noise~\cite{5,7}. As described in Definition~\ref{mean sampling}, it uses the mean of $m$ independent evaluations to approximate the true fitness $f(x)$, where $m$ is called the sample size. By mean sampling, the output $\bar{f}(x)$ is close to the mathematical expectation of $f^\mathrm{n}(x)$.
As described in Definition~\ref{median sampling}, median sampling takes the median of $m$ independent evaluations to approximate the true fitness $f(x)$. By median sampling, the output $\hat{f}(x)$ is close to the 2-quantile of $f^\mathrm{n}(x)$, namely $\mathrm{P}(f^\mathrm{n}(x)\le \hat{f}(x))\approx 1/2$.
\begin{definition}[Mean Sampling]\label{mean sampling}
The objective value of $x$ is evaluated independently $m$ times, then
$$
	\bar{f}(x)=\sum^{m}_{i=1} \frac{f^{\mathrm{n}}_i(x)}{m}
$$
is output, where  $f^{\mathrm{n}}_1(x),f^{\mathrm{n}}_2(x),\ldots,f^{\mathrm{n}}_m(x)$ denote $m$ noisy fitness values.
\end{definition}
\begin{definition}[Median Sampling]\label{median sampling}
The objective value of $x$ is evaluated independently $m$ times, 
then 
	\begin{align*}
	\hat{f}(x)=\begin{cases}
	f^{\mathrm{n}}_{i_{(m+1)/2}}(x) & \text{if $m$ is odd},\\
	\big(f^{\mathrm{n}}_{i_{m/2}}(x)+f^{\mathrm{n}}_{i_{m/2+1}}(x)\big)/2 & \text{if $m$ is even}.
	\end{cases}
	\end{align*}
is output, where $f^{\mathrm{n}}_{i_1}(x)\le f^{\mathrm{n}}_{i_2}(x)\le\ldots\le f^{\mathrm{n}}_{i_m}(x)$ denote the ordered noisy fitness values. 
\end{definition}
When mean (or median) sampling is used, line~5 in Algorithm~\ref{(1+1)-EA} becomes ``$\bar{f}(\cdot) $" (or ``$\hat{f}(\cdot)$"). For both of the sampling strategies, $m=1$ means that sampling is not used.

\subsection{Analysis Tool}
It is straightforward to model the evolutionary optimization procedure as a Markov chain $\{\xi_t\}^{+\infty}_{t=0}$, because the subsequent procedure only depends on the current state. For (1+1)-EA optimizing OM, we can simply set the chain's state space as $\{0,1\}^n$ and the optimal state as $1^n$ (namely $\xi_t \in \mathcal{X}=\{0,1\}^n$ and $\mathcal{X}^*=\{1^n\}$). 
The \emph{first hitting time} (FHT) of $\{\xi_t\}^{\infty}_{t=0}$ is $\tau=\min\{t \mid \xi_{t} \in \mathcal{X}^*,t\geq0\}$. If the chain's initial state is $\xi_{0}=x$, then its \emph{expected FHT} (EFHT)  is denoted as $\mathrm{E}(\tau \mid \xi_0=x)=\sum^{\infty}_{t=0} t\cdot\mathrm{P}(\tau=t \mid \xi_0=x)$. If $\xi_0$ obeys a distribution $\pi_{0}$ (denoted as $\xi_{0}\!\sim\! \pi_0$), then its EFHT is defined as $\mathrm{E}(\tau \mid \xi_{0}\!\sim\! \pi_0) \!= \!\sum_{x\in \mathcal{X}} \pi_{0}(x)\mathrm{E}(\tau \mid \xi_{0}=x)$. 
For (1+1)-EA, the initial solution is evaluated once, then in each iteration, the parent solution and the offspring solution both need to be evaluated. Note that the initial solution is generated randomly from $\{0,1\}^n$, thus the ERT of (1+1)-EA is $1+2\cdot \mathrm{E}(\tau \mid \xi_{0} \sim \pi_0)$, where $\pi_0$ denotes the uniform distribution.
For (1+1)-EA using sampling, the ERT becomes $m+2m\cdot \mathrm{E}(\tau \mid \xi_{0} \sim \pi_0)$, because it needs to perform $m$ independent evaluations for each solution. 

As presented in Theorem~\ref{additive-drift}, the additive drift theorem aims to derive upper bounds of EFHT. To use the approach, first we need to design a function $V(\cdot)$ as a measurement for the difference between a state and the optimal state, and $V(\cdot)$ should satisfy $V(x)=0$ for any optimal $x$ and $V(x)>0$ otherwise. Next, we need deriving a lower bound $c$ for $\mathrm{E}(V(\xi_t)-V(\xi_{t+1})|\xi_t)$, i.e., the progress towards $\mathcal{X^*}$ in each generation. Finally, we can upper bound EFHT through dividing $V(\xi_0)$ by $c$.  When the context is clear, $V(\xi_t)/V(\xi_{t+1})$ will be briefly denoted as $V_t/V_{t+1}$. 
\begin{theorem}[Additive drift~\cite{39}] \label{additive-drift}
Given $\{\xi_t\}_{t=0}^{+\infty}$ and $V(\cdot)$, if $\exists c>0$ such that $\forall t\ge 0$ and $\forall \xi_t$ with $V_t>0$, 
	\[
	\mathrm{E}(V_t-V_{t+1}|\xi_t)\ge c,
	\]
	then $\mathrm{E}(\tau|\xi_0)\le V(\xi_0)/c$.
\end{theorem}

\section{The Robustness of Median Sampling Against Onebit Noise}
Onebit noise is commonly used in theoretical analyses~\cite{17,18,26,27}. With probability $p$, it changes a uniformly selected bit in $x$ before $x$ is evaluated.
For OM under such noise model, the ERT of (1+1)-EA is superpolynomial for $p=\omega(\log n/n)$ \cite{17}; the ERT is polynomial $\forall p\in [0,1]$ if using mean sampling with $m=4n^3$~\cite{27}. Theorem~\ref{theo-median-onebit} shows that the ERT is polynomial if using median sampling with $m=2n^3+1$, which illustrates that median sampling can efficiently tackle noise.
\begin{definition}[Onebit Noise]\label{onebit-noise}
Suppose $f^{\mathrm{n}}(\cdot)/f(\cdot)$ denotes the noisy/true objective function. Then
\begin{align*}
\mathrm{P}(f^{\mathrm{n}}(x)=f(x))=1-p,\quad
\mathrm{P}(f^{\mathrm{n}}(x)=f(y))=p,
\end{align*}
where $p \in [0,1]$ and $y$ is derived by changing a uniformly selected bit in $x$.
\end{definition}

To prove Theorem~\ref{theo-median-onebit}, we present Lemma~\ref{lem-median-onebit} to analyze $\hat{f}(x)$ under onebit noise by taking a sample size of $2n^3+1$. It intuitively means $\hat{f}(x)$ is close to the 2-quantile of $f^\mathrm{n}(x)$ w.h.p.
\begin{lemma}\label{lem-median-onebit}
 Under onebit noise, if median sampling with $m=2n^3+1$ is used, then
\begin{enumerate}
	\item[(i)] if $p\cdot \frac{n-|x|_0}{n}\ge \frac{1}{2}+\frac{\Omega(1)}{n}$, then  $\mathrm{P}(\hat{f}(x)=n-|x|_0-1)\ge 1-e^{-\Omega(n)}$;
	\item[(ii)] if $p\cdot \frac{|x|_0}{n}\ge \frac{1}{2}+\frac{\Omega(1)}{n}$, then $\mathrm{P}(\hat{f}(x)=n-|x|_0+1)\ge 1-e^{-\Omega(n)}$;
	\item[(iii)] if $p\cdot \frac{|x|_0}{n}\le \frac{1}{2}-\frac{\Omega(1)}{n}$, then $\mathrm{P}(\hat{f}(x)\le n-|x|_0)\ge 1-e^{-\Omega(n)}$;
	furthermore, if 
	$p\cdot \frac{n-|x|_0}{n}\le \frac{1}{2}-\frac{\Omega(1)}{n}$ also holds, then $\mathrm{P}(\hat{f}(x)=n-|x|_0)\ge 1-e^{-\Omega(n)}$.
\end{enumerate}
\end{lemma}
\begin{proof}
First we consider (i). Suppose $ |x|_0=i$ and $(n-i)p/n\ge 1/2+c/n$ for a constant $c$. Suppose $s$ denotes the number of noisy evaluations satisfying $f^{\mathrm{n}}(x)=n-1-i $ in $m$ independent noisy evaluations. Observe that in each evaluation, $\mathrm{P}(f^{\mathrm{n}}(x)=n-1-i)=(n-i)p/n$, thus  $\mathrm{E}(s)=m(n-i)p/n \ge m(1/2+c/n) \ge m/2+cn^2$. Then we get
\begin{equation}
\begin{aligned}
\mathrm{P}\left(s\le m/2\right)
& = \mathrm{P}\left(s\le m/2+cn^2-cn^2\right)\le \mathrm{P}\left(s\le \mathrm{E}(s)-cn^2\right) \\
& \le \mathrm{P}\left(|s-\mathrm{E}(s)|\ge cn^2\right) \le 2e^{-2c^2n^4/m}=e^{-\Omega(n)},
\end{aligned}
\end{equation}
where the last inequality is derived according to Hoeffding's inequality. Therefore, 
\begin{equation}\label{eq-r}
\mathrm{P}(s> m/2)\ge 1-e^{-\Omega(n)}.
\end{equation}
By the definition of median sampling,  $\mathrm{P}(\hat{f}(x)=n-|x|_0-1)\ge 1-e^{-\Omega(n)}$, thus the claim holds. We can similarly prove (ii).\\
Now we consider (iii). Under onebit noise, $f^\mathrm{n}(x)$ can take at most three values (i.e., $n-1-|x|_0$, $n-|x|_0$, $n+1-|x|_0$), thus $\hat{f}(x)$ can only take one of the three values by the definition of median sampling and $m=2n^3+1$. Note that $\mathrm{P}(f^{\mathrm{n}}(x)\le n-|x|_0)=1-\mathrm{P}(f^{\mathrm{n}}(x)=n+1-|x|_0)=1-|x|_0p/n\ge 1/2+\Omega(1)/n$, then similar to case~(i),  $\mathrm{P}(\hat{f}(x)\le n-|x|_0)\ge 1-e^{-\Omega(n)}$. \\
Then we consider the ``furthermore" clause. By $\mathrm{P}(f^{\mathrm{n}}(x)\ge n-|x|_0)=1-\mathrm{P}(f^{\mathrm{n}}(x)=n-1-|x|_0)=1-(n-|x|_0)p/n\ge 1/2+\Omega(1)/n$, we also derive $\mathrm{P}(\hat{f}(x)\ge n-|x|_0)\ge 1-e^{-\Omega(n)}$. Thus, $\mathrm{P}(\hat{f}(x)= n-|x|_0)=1-\mathrm{P}(\hat{f}(x)>n-|x|_0)-\mathrm{P}(\hat{f}(x)<n-|x|_0)\ge 1-e^{-\Omega(n)}$, i.e., the claim holds.\\
Combining the above analysis, the Lemma holds.
\qed $\blacksquare$
\end{proof}

\begin{theorem}\label{theo-median-onebit}
For OM under onebit noise, the ERT of (1+1)-EA employing median sampling with $m=2n^3+1$ is polynomial.
\end{theorem}
\begin{proof}
The main idea is applying Theorem~\ref{additive-drift}. We consider three cases for $p$ and in each case, we will design a distance function $V(x)$ and we need to examine $\mathrm{E}(V_t-V_{t+1} \mid \xi_t=x)$ for $x\neq 1^n$. Suppose $|x|_0=i$, $1 \leq i \leq n$. For ease of notation, let $\mathrm{P}_{\rm mut}(x,z)=\mathrm{P}$($z$ is mutated from $x$), and  $\mathrm{P}_{\rm acc}(x,z)=\mathrm{P}(\hat{f}(z) \geq \hat{f}(x))$. For ease of analysis, the drift is divided into $\mathrm{E}_1$ and $\mathrm{E}_2$. That is,
\begin{equation}	\mathrm{E}(V_t-V_{t+1}\mid \xi_t=x)=\mathrm{E}_1 - \mathrm{E}_2,\vspace{-0.3em}
\end{equation}
where
\begin{equation}
\mathrm{E}_1=\sum_{|z|_0<i} \mathrm{P}_{\rm mut}(x,z)\cdot \mathrm{P}_{\rm acc}(x,z) \cdot (V(x)-V(z)),\vspace{-0.5em}
\end{equation}
\begin{equation}
\mathrm{E}_2=\sum_{|z|_0>i} \mathrm{P}_{\rm mut}(x,z)\cdot \mathrm{P}_{\rm acc}(x,z) \cdot (V(z)-V(x)).
\end{equation}

(1) $p\le n/(2(n+1))$. $V(x)$ is designed to be $|x|_0$, namely the number of 0s in $x$.\\
For $\mathrm{E}_1$, we consider mutating only one zero bit in $x$  (namely $|z|_0=i-1$), and its probability is $i/n\cdot (1-1/n)^{n-1}\ge i/(en)$. Then $z$ will replace $x$ if $\hat{f}(z)=n+1-i$ and $\hat{f}(x)=n-i$. Conditions of  Lemma~\ref{lem-median-onebit}-(iii) hold because $p\le n/(2(n+1))=1/2-1/(2(n+1))$, then
\begin{equation}\label{pac-iii}
\mathrm{P}_{\rm acc}(x,z)
\ge \mathrm{P}(\hat{f}(z)=n+1-i,\hat{f}(x)=n-i)\ge 1-e^{-\Omega(n)}.
\end{equation} 
Thus, 
\begin{equation}\label{onebit-E+}
	\mathrm{E}_1\ge \frac{i}{en}\cdot (1-e^{-\Omega(n)})\ge \frac{1}{3n},
\end{equation}
where the last inequality is by $n$ is large enough.\\
For $\mathrm{E}_2$, we consider the increase of 0s. For $z$ satisfying $|z|_0>i$, accepting it implies $\hat{f}(z)\neq n-|z|_0$ or $\hat{f}(x)\neq n-i$. 
Note that conditions of (iii) in Lemma~\ref{lem-median-onebit} are satisfied, then we have
\begin{equation}\label{pac-iii-E-}
\mathrm{P}_{\rm acc}(x,z)\le e^{-\Omega(n)}.
\end{equation}
Thus, 
\begin{equation}\label{onebit-E-}
\mathrm{E}_2\le (n-i)\cdot e^{-\Omega(n)}\le e^{\log n}\cdot e^{-\Omega(n)}=e^{-\Omega(n)},
\end{equation}
where the last equality is by $n$ is large enough.\\
Subtract $ \mathrm{E}_2 $ from $\mathrm{E}_1 $, we get
\begin{equation}\label{onebit-E}
\mathrm{E}\left(V_t-V_{t+1}\mid \xi_t=x\right)\ge \Omega\left(\frac{1}{n}\right),
\end{equation}
where the last equation derives from large enough $n$. Therefore, by Theorem~\ref{additive-drift}, $\mathrm{E}(\tau|\xi_0)\le n/\Omega(1/n)=O(n^2)$, because $ V(x)\le n$. Note that each iteration needs $2m=4n^3+2$ fitness evaluations, we can derive a   polynomial ERT.

(2) $n/(2(n+1))<p<n/(n+7)$. The proof procedure is similar to case~(1), but the $V(x)$ is more complicated because the effect of the noise on a solution $x$ may vary as $|x|_0$ changes. The distance function is as follows:
\begin{align*}
V(x)=\begin{cases}
i &  \text{if $i>\frac{n}{2p}+3$ or $n-\frac{n}{2p}+3<i<\frac{n}{2p}-3$ or $i< \max\{1,n-\frac{n}{2p}-3\}$},\\
\frac{n}{2p} & \text{if $\frac{n}{2p}-3\le i\le \frac{n}{2p}+3$},\\
n-\frac{n}{2p}+2 & \text{if $\max\{1,n-\frac{n}{2p}-3\}\le i\le n-\frac{n}{2p}+3$}.
\end{cases}
\end{align*}
We consider five cases for $i$.\\
(2a) $i>n/(2p)+3$. \\
For $\mathrm{E}_1$, we also consider mutating only one zero bit in $x$, then $\mathrm{P}_{\rm mut}(x,z)\ge i/(en)$. Note that $|z|_0>n/(2p)+2$, thus $pi/n > p|z|_0/n\ge 1/2+2p/n$. 
By (ii) in Lemma~\ref{lem-median-onebit}, 
\begin{equation}\label{pac-ii}
\mathrm{P}_{\rm acc}(x,z)
\ge \mathrm{P}(\hat{f}(z)=n+2-i,\hat{f}(x)=n+1-i)\ge 1-e^{-\Omega(n)}.
\end{equation}
If $|z|_0>n/(2p)+3$, then $V(x)-V(z)=1$; else $V(x)-V(z)=i-n/(2p)>3$. Thus, $V(x)-V(z)\ge 1$ and  $\mathrm{E}_1= \Omega(1/n)$. \\
Now we consider $\mathrm{E}_2$. For $z$ satisfying $|z|_0>i$, $p|z|_0/n>pi/n>1/2+3p/n\ge 1/2+3/(2(n+1))$. Thus, by (ii) in Lemma~\ref{lem-median-onebit},
\begin{equation}\label{pac-ii-E-}
\mathrm{P}(\hat{f}(z)=n+1-|z|_0,\hat{f}(x)=n+1-i)\ge 1-e^{-\Omega(n)},
\end{equation} 
implying $\mathrm{P}_{\rm acc}(x,z)\le e^{-\Omega(n)}$. Accordingly, $\mathrm{E}_2\le e^{-\Omega(n)}$.\\
(2b) $n/(2p)-3\le i\le n/(2p)+3$. Note that $i\ge n/(2p)-3\ge (n+7)/2-3=\Omega(n)$.\\
First we consider the positive drift $\mathrm{E}_1$. By $n/(2p)-3-(n-n/(2p)+3)=n/p-n-6>1$, there always exists some $z$ such that $n-n/(2p)+3<|z|_0=\lceil n/(2p)-3\rceil-1$ and such $z$ can be mutated from $x$ by flipping at most seven 0s. Thus,
$$\mathrm{P}_{\rm mut}(x,z)\ge \binom{i}{7}\left(1-\frac{1}{n}\right)^{n-7} \left(\frac{1}{n}\right)^7=\Omega(1).$$
Note that $p(n-|z|_0)/n<1/2-3p/n$ and $p|z|_0/n<1/2-3p/n$, thus $\mathrm{P}(\hat{f}(z)=n-i)\ge 1-e^{-\Omega(n)}$ by (iii) in Lemma~\ref{lem-median-onebit}. Therefore, $z$ will replace $x$  with probability $1-e^{-\Omega(n)}$. Moreover, 
$$V(x)-V(z)\ge n/(2p)-|z|_0>n/(2p)-(n/(2p)-3)=3,$$ we have $\mathrm{E}_1= \Omega(1)$.\\ 
For $\mathrm{E}_2$, we consider $z$ with $|z|_0>i$. If $i>n/(2p)+1$, Eq.~\eqref{pac-ii-E-} holds and we have $\mathrm{P}_{\rm acc}(x,z)\le e^{-\Omega(n)}$, then we get $$\mathrm{E}_2\le (n-i)\cdot e^{-\Omega(n)}\le e^{\log n}\cdot e^{-\Omega(n)}=e^{-\Omega(n)},$$
where the last equality is by $n$ is large enough.\\
If $i\le n/(2p)+1$, then any $z$ satisfying $|z|_0\ge i+3$ will never be accepted under onebit noise. For $z$ satisfying $|z|_0\le i+2$, we have $|z|_0\le n/(2p)+3$, thus $V(z)=n/(2p)=V(x)$. Then $\mathrm{E}_2=0$. Combining the two cases for $z$, we get $\mathrm{E}_2\le e^{-\Omega(n)}$.\\
(2c) $n-n/(2p)+3<i<n/(2p)-3$. First we examine $\hat{f}(x)$. Note that $p(n-i)/n<1/2-3p/n$ and $pi/n<1/2-3p/n$, thus $\mathrm{P}(\hat{f}(x)=n-i)\ge 1-e^{-\Omega(n)}$ by (iii) in Lemma~\ref{lem-median-onebit}.\\
For $\mathrm{E}_1 $, we consider mutating only one zero bit in $x$, namely $|z|_0=i-1>n-n/(2p)+2$. Note that $p(n-|z|_0)/n<1/2-2p/n$ and $p|z|_0/n<1/2-3p/n$, we can derive Eq.~\eqref{pac-iii} and $\mathrm{E}_1= \Omega(1/n)$.\\
For $\mathrm{E}_2$, we consider two cases for $z$ satisfying $|z|_0>i$. If $|z|_0\ge i+2$, accepting $z$ implies $\hat{f}(x)=n-i-1$. Thus, $\mathrm{P}_{\rm acc}(x,z)\le e^{-\Omega(n)}$. If $|z|_0=i+1$, then $p(n-|z|_0)/n<1/2-4p/n$ and $p|z|_0/n<1/2-2p/n$, thus $\mathrm{P}(\hat{f}(z)=n-|z|_0)\ge 1-e^{-\Omega(n)}$. Then Eq.~\eqref{pac-iii-E-} still holds, i.e., $\mathrm{P}_{\rm acc}(x,z)\le e^{-\Omega(n)}$. Combining the two cases, $\mathrm{E}_2\le e^{-\Omega(n)}$.\\
(2d) $\max\{1,n-n/(2p)-3\}\le i\le n-n/(2p)+3$.\\
First we consider the positive drift $\mathrm{E}_1 $. Note that $$n-\frac{n}{2p}+3-\max\{1,n-\frac{n}{2p}-3\}\le n-\frac{n}{2p}+3-\left(n-\frac{n}{2p}-3\right)\le 6,$$
$x$ can generate an offspring $z$ with $|z|_0<\max\{1,n-n/(2p)-3\}$ by flipping at most seven bits, whose probability is at least $\Omega(1/n^7)$. Then we examine $\mathrm{P}_{\rm acc}(x,z)$. Because 
$$\frac{pi}{n}\le p\left(n-\frac{n}{2p}+3\right) \frac{1}{n}=p\frac{n+3}{n}-\frac{1}{2}\le \frac{n+3}{n+7}-\frac{1}{2}=\frac{1}{2}-\frac{4}{n+7},$$ we derive $\mathrm{P}(\hat{f}(x)\le n-i)\ge 1-e^{-\Omega(n)}$ by (iii) in Lemma~\ref{lem-median-onebit}. Note that $\hat{f}(z)\ge n-i$ always holds under onebit noise, we get $\mathrm{P}_{\rm acc}(x,z)\ge 1-e^{-\Omega(n)}$. If $1\le n-n/(2p)-3$, we get $$V(x)-V(z)\ge n-\frac{n}{2p}+2-\left(n-\frac{n}{2p}-3\right)\ge 5;$$ 
else $|z|_0=0$ and  
$$V(x)-V(z)\ge n-\frac{n}{2p}+2>1.$$ Thus, $\mathrm{E}_1=\Omega(1/n^7)$.\\
For $\mathrm{E}_2$, it is only necessary to take $z$ satisfying $|z|_0\le i+2$ into account because $z$ with $|z|_0\ge i+3$ will be rejected. If $i\ge n-n/(2p)+1$, we have $n-n/(2p)+1\le i< |z|_0\le n-n/(2p)+5$. Thus, 
$$p\cdot\frac{n-|z|_0}{n}<p\cdot \frac{n-i}{n}\le \frac{1}{2}-\frac{p}{n}$$ and $$p\cdot \frac{i}{n}< p\cdot \frac{|z|_0}{n}\le \frac{(n+5)p}{n}-\frac{1}{2}<\frac{n+5}{n+7}-\frac{1}{2}=\frac{1}{2}-\frac{2}{n+7},$$ which implies that Eq.~\eqref{pac-iii-E-} holds by (iii) in Lemma~\ref{lem-median-onebit}.
If $i< n-n/(2p)+1$, we have $|z|_0<n-n/(2p)+3$. By $V(x)=V(z)=n-n/(2p)+2$, we have $\mathrm{E}_2=0$. Combining the  two cases, $\mathrm{E}_2\le e^{-\Omega(n)}$.\\
(2e) $i<\max\{1,n-n/(2p)-3\}$. If $1\ge n-n/(2p)-3$, then $i=0$, thus we only need to consider that $1< n-n/(2p)-3$, namely $i<n-n/(2p)-3$.\\
For $\mathrm{E}_1$, we consider mutating only one zero bit in $x$, i.e., $|z|_0=i-1$. Similar to the above analysis, $\mathrm{P}_{\rm mut}(x,z)\ge 1/(en)$. Note that $p(n-i)/n>1/2+3p/n$ , we derive $\mathrm{P}(\hat{f}(x)=n-i-1)\ge 1-e^{-\Omega(n)}$ by (i) in Lemma~\ref{lem-median-onebit}. Thus, $\mathrm{P}_{\rm acc}(x,z)\ge 1-e^{-\Omega(n)}$. Note that $V(x)-V(z)=i-|z|_0=1$, thus $\mathrm{E}_1= \Omega(1/n)$.\\
For $\mathrm{E}_2$, it is only necessary to take $z$ with $|z|_0\le i+2<n-n/(2p)-1$ into account. Note that $p(n-|z|_0)/n>1/2+p/n$, thus by Lemma~\ref{lem-median-onebit},  $\mathrm{P}(\hat{f}(z)=n-|z|_0-1)=1-e^{-\Omega(n)}$, implying that $\mathrm{P}_{\rm acc}(x,z)\le e^{-\Omega(n)}$. Then we have $\mathrm{E}_2\le e^{-\Omega(n)}$.\\
Combining the five cases, we have $\mathrm{E}_1\ge \Omega(1/n^7)$ and $\mathrm{E}_2\le e^{-\Omega(n)}$. By subtracting $ \mathrm{E}_2 $ from $\mathrm{E}_1 $, Eq.~\eqref{onebit-E} becomes
\begin{equation}
\mathrm{E}\left(V_t-V_{t+1}\mid \xi_t=x\right)=\Omega\left(\frac{1}{n^7}\right),
\end{equation}
and we can also derive a polynomial ERT.

(3) $p\ge n/(n+7)$. The effect of the noise changes when the level of the noise changes. Accordingly, we need to design a new distance function:
\begin{align*}
V(x)=\begin{cases}
i &  \text{if $i>\frac{n}{2p}+3$ or $i<n-\frac{n}{2p}-3$},\\
\frac{n}{2} & \text{if $n-\frac{n}{2p}-3\le i\le \frac{n}{2p}+3$}.
\end{cases}
\end{align*}
Next we consider three cases for $i$.\\
(3a) $i>n/(2p)+3$. The proof procedure is the same as case~(2a), except that ``$V(x)-V(z)=i-n/(2p)>3$" changes to $V(x)-V(z)=i-n/2>n/(2p)+3-n/2\ge 3$. We derive  $\mathrm{E}_1=\Omega(1/n)$ and $\mathrm{E}_2\le e^{-\Omega(n)}$.\\
(3b) $n-n/(2p)-3\le i\le n/(2p)+3$. Note that $i\ge n-n/(2p)-3=\Omega(n)$.\\
First we consider the positive drift $\mathrm{E}_1$. There exists some $z$ with $|z|_0=\lceil n-n/(2p)-3\rceil-1$ and such $z$ can be mutated from $x$ by flipping at most $n/(2p)+3-(n-n/(2p)-3)+1=n/p-n+6+1\le 14$ 0s. Thus,  $$\mathrm{P}_{\rm mut}(x,z)\ge \binom{i}{14}\left(1-\frac{1}{n}\right)^{n-14} \left(\frac{1}{n}\right)^{14}=\Omega(1).$$ If $i\le n-n/(2p)-1$, then $p(n-i)/n\ge 1/2+p/n$. Thus,  $\mathrm{P}_{\rm acc}(x,z)\ge \mathrm{P}(\hat{f}(x)=n-i-1)\ge 1-e^{-\Omega(n)}$ by (i) in Lemma~~\ref{lem-median-onebit}.
If $i> n-n/(2p)-1$, it can be verified that $z$ will always be accepted under onebit noise. Note that 
$$V(x)-V(z)\ge n/2-|z|_0\ge n/2-(n-n/(2p)-3)\ge 3.$$ Thus, we have $\mathrm{E}_1= \Omega(1)$.\\ 
For $\mathrm{E}_2$, the proof procedure is the same as that of case~(2b), except that ``$V(z)=n/(2p)=V(x)$" changes to $V(z)=n/2=V(x)$. Thus, we get $\mathrm{E}_2\le e^{-\Omega(n)}$.\\
(3c) $i<n-n/(2p)-3$. The analysis for $\mathrm{E}_1$ and $\mathrm{E}_2$ is the same as that of case~(2e), then we have $\mathrm{E}_1= \Omega(1/n)$ and $\mathrm{E}_2\le e^{-\Omega(n)}$.\\
Combining the three cases, we have $\mathrm{E}_1= \Omega(1/n)$ and $\mathrm{E}_2\le e^{-\Omega(n)}$. Subtract $ \mathrm{E}_2 $ from $\mathrm{E}_1 $, we get
\begin{equation}
\mathrm{E}\left(V_t-V_{t+1}\mid \xi_t=x\right)=\Omega\left(\frac{1}{n}\right),
\end{equation}
and we can also derive a polynomial ERT.
\qed $\blacksquare$
\end{proof}
By the above proof, we can give an intuitive explanation for the effectiveness of median sampling. For $x$ and $z$ which satisfy $f(x)\!>\!f(z)$, when the 2-quantile of $f^\mathrm{n}(x)$ is larger than that of $f^\mathrm{n}(z)$, $x$ will be estimated better than $z$ by median sampling w.h.p., implying a correct comparison.

\section{Cases Where Median Sampling is Better than Mean Sampling}\label{compare}
For OM under segmented noise (Definition~\ref{def-seg-noise}), we show that (1+1)-EA
equipped with median sampling can do better than (1+1)-EA using mean sampling. The segmented noise is from \cite{25}, but we make a little modification to simplify the analysis. As presented in Definition~\ref{def-seg-noise}, the noisy evaluation of a solution $x$ can be divided into three segments. The objective evaluation is  accurate in the first segment, but inaccurate in other segments  because of noise.
We show that for OM under segmented noise, the ERT of (1+1)-EA using mean sampling is exponential (i.e., Theorem~\ref{theo-mean-segment}); and the ERT of (1+1)-EA employing median sampling with $m=2n^3+1$ is polynomial (i.e., Theorem~\ref{theo-median-segment}). The analyses show that  median sampling can be better if the 2-quantile increases with the true fitness.
\begin{definition}\label{def-seg-noise}
	 $\forall x\in \{0,1\}^n$, its noisy objective $f^{\mathrm{n}}(\cdot)$ is defined as follows:\\
	(1) if $|x|_0>\frac{n}{50}$, 
	$f^{\mathrm{n}}(x)=n-|x|_0$;\\
	(2) if $\frac{n}{100}< |x|_0\leq \frac{n}{50}$, \begin{align*}
	\mathrm{P}(f^{\mathrm{n}}(x)=
	n-|x|_0) =\frac{1}{2}+\frac{1}{n},\quad 
	\mathrm{P}(f^{\mathrm{n}}(x)=3n+|x|_0)=\frac{1}{2}-\frac{1}{n};
	\end{align*}
	(3) if $|x|_0\leq \frac{n}{100}$,
	 \begin{align*}
	\mathrm{P}(f^{\mathrm{n}}(x)=	4n(n-|x|_0))=1-\frac{1}{n},\quad 
	\mathrm{P}(f^{\mathrm{n}}(x)=(2n+|x|_0)^3)=\frac{1}{n};
	\end{align*}
	where $n/100\in \mathbb{N}^+$.
\end{definition}

Theorem~\ref{theo-mean-segment} shows that mean sampling fails under segmented noise and the reason is similar to that found in \cite{25}. Consider $x$ and $z$ satisfying $|z|_0=|x|_0+1$. In segment~(2), a small sample cannot eliminate the impact of noise, and $\mathrm{P}(\bar{f}(x) \leq \bar{f}(z))$ is still very large. In segment~(3), the expected gap between $f^{\mathrm{n}}(z)$ and $f^{\mathrm{n}}(x)$ is positive. Therefore, a larger sample size will enlarge $\mathrm{P}(\bar{f}(x) \leq \bar{f}(z))$ and performs worse; moreover, no medium sample size makes a good tradeoff. Therefore, mean sampling fails. Its rigorous proof can be derived directly from Theorem~5.2 in \cite{25}, because the change of noise doesn't affect the proof.
\begin{theorem}\label{theo-mean-segment}
	For OM under segmented noise, the ERT of (1+1)-EA employing mean sampling is exponential.
\end{theorem}

To prove Theorem~\ref{theo-median-segment}, Lemma~\ref{lemma-upper} is used. This lemma can upper bound the runtime, when the true better solution has a large probability to be recognized as better. Note that $x^j$ denotes some solution with $j$ 0s, and $F(\cdot)$ denotes the estimated fitness of a solution.
\begin{lemma}[\!\!\cite{18}]\label{lemma-upper}
The EFHT of (1+1)-EA solving noisy OM is polynomial if
	\begin{align}
	&\forall 0<i\le j:\mathrm{P}(F(x^j)\ge F(x^{i-1})) \le \frac{\log n}{15n}.
	\end{align}
	
\end{lemma}
We also present Lemma~\ref{lem-median-segment} to analyze $\hat{f}(x)$ under segmented noise by taking a sample size of $2n^3+1$. 
\begin{lemma}\label{lem-median-segment}
Under segmented noise, if median sampling with $m=2n^3+1$ is used, then
$\mathrm{P}(\hat{f}(x)=n-|x|_0)=1-e^{-\Omega(n)}$ if $\frac{n}{100}< |x|_0\leq \frac{n}{50}$; $\mathrm{P}(\hat{f}(x)=4n(n-|x|_0))=1-e^{-\Omega(n)}$ if $|x|_0\leq \frac{n}{100}$.
\end{lemma}
\begin{proof}
	The main procedure is analogous to Lemma~\ref{lem-median-onebit}. If $\frac{n}{100}< |x|_0\leq \frac{n}{50}$, suppose there are $s$ noisy evaluations where $f^{\mathrm{n}}(x)=n-|x|_0 $ in $m$ independent noisy evaluations. Then Eq.~\eqref{eq-r} also holds and  $\mathrm{P}(\hat{f}(x)=n-|x|_0)=1-e^{-\Omega(n)}$. If $|x|_0\leq \frac{n}{100}$, we similarly have $\mathrm{P}(\hat{f}(x)=4n(n-|x|_0))=1-e^{-\Omega(n)}$. Thus, the lemma holds.
\qed $\blacksquare$
\end{proof}
\begin{theorem}\label{theo-median-segment}
	For OM under segmented noise, the ERT of (1+1)-EA employing median sampling with $m=2n^3+1$ is polynomial.
\end{theorem}
\begin{proof}
	The main idea is applying Lemma~\ref{lemma-upper}. Given $0<i\le j$, let $g=\hat{f}(x^j)-\hat{f}(x^{i-1})$.
	To analyze $\mathrm{P}(g\ge 0)$, we consider four cases for $i$. \\
	(1) $i>\frac{n}{50}$. Note that $f^{\mathrm{n}}(x^j)=f(x^j)$ and $f^{\mathrm{n}}(x^{i-1})$ is larger by considering $i-1>\frac{n}{50}$ and $i-1\le\frac{n}{50}$, respectively. Therefore, we get $\mathrm{P}(g\ge 0) =0$. \\
	(2) $\frac{n}{100}+1< i\le\frac{n}{50}$. If $j>\frac{n}{50}$, we have $ \mathrm{P}(g\ge 0) =0 $, because $f^{\mathrm{n}}(x^{j})=n-j$ and $f^{\mathrm{n}}(x^{i-1})\ge n-i+1>n-j$. If $j\le \frac{n}{50}$, by Lemma~\ref{lem-median-segment}, we have $$\mathrm{P}(g<0)\ge \mathrm{P}\big(\hat{f}(x^j)=n-j,\hat{f}(x^{i-1})=n-i+1\big)=1-e^{-\Omega(n)}.$$
	(3) $i\le\frac{n}{100}+1$. The analysis is analogous to case~(2). If $j>\frac{n}{100}$, then $\mathrm{P}(g\ge 0)=0$. If $j\le\frac{n}{100}$, then 
	$$\mathrm{P}(g<0)\ge \mathrm{P}\big(\hat{f}(x^j)=4n(n-j),\hat{f}(x^{i-1})=4n(n-i+1)\big)=1-e^{-\Omega(n)}.$$
	Combining the three cases, we have shown $\forall 0<i\le j: \mathrm{P}(g\ge 0) \leq \log n/(15n)$ for sufficiently large $n$. Then, by Lemma~\ref{lemma-upper}, the EFHT is polynomial. In each iteration, the algorithm needs $2m=4n^3+2$ evaluations, thus the ERT is polynomial.
\qed $\blacksquare$
\end{proof}

\section{Cases Where Mean Sampling is Better than Median Sampling}\label{median-fail}
For OM under partial noise (Definition~\ref{def-part-noise}), we show that (1+1)-EA
using median sampling is sometimes worse than (1+1)-EA using mean sampling. For partial noise presented in Definition~\ref{def-part-noise},
a false objective value is returned when $|x|_0<n/2$. We prove that for OM under partial noise, the ERT of (1+1)-EA employing mean sampling with $m=n^3$ is polynomial (i.e., Theorem~\ref{theo-mean-part}); and the ERT of (1+1)-EA employing median sampling is exponential (i.e., Theorem~\ref{theo-median-part}). The analyses suggest that median sampling may fail if the 2-quantile of the noisy fitness doesn't increase with the true objective value, and it is better to choose other strategies. 
\begin{definition}\label{def-part-noise}
$\forall x\in \{0,1\}^n$, its noisy objective $f^{\mathrm{n}}(\cdot)$ is defined as follows:\\
	(1) if $|x|_0\ge \frac{n}{2}$, $f^{\mathrm{n}}(x)=
	n-|x|_0$;\\
	(2) if $|x|_0< \frac{n}{2}$, 
	$$\mathrm{P}(f^{\mathrm{n}}(x)=|x|_0/2)=2/3, \quad
	\mathrm{P}(f^{\mathrm{n}}(x)=2(n-|x|_0))=1/3.$$ 
\end{definition}
\begin{theorem}\label{theo-mean-part}
	For OM under partial noise, the ERT of (1+1)-EA employing mean sampling with $m=n^3$ is polynomial.
\end{theorem}
\begin{proof}
The main idea is applying Lemma~\ref{lemma-upper}. Given $0<i\le j$, let $g=\bar{f}(x^{i-1})-\bar{f}(x^j)$. To analyze $\mathrm{P}(g\ge 0)$, we classify $i$ into two cases.\\
	(1)  $i\ge \frac{n}{2}+1$. We have $\bar{f}(x^{i-1})=n-i+1$ and $\bar{f}(x^j)=n-j$, thus $\mathrm{P}(g\le 0)=0$.\\
	(2) $i<\frac{n}{2}+1$. First we need to derive $\mu:= \mathrm{E}(g)$. Note that $\mathrm{E}(\bar{f}(x^{i-1}))=\mathrm{E}(f^{\mathrm{n}}(x^{i-1}))=(i-1)/2\cdot 2/3+2(n-i+1)\cdot 1/3=(2n-i+1)/3$. We classify $j$ into two cases. (a) If $j\ge\frac{n}{2}$, then $\mathrm{E}(\bar{f}(x^j))=n-j$, thus 
	\begin{equation*}
	\mathrm{E}(g)=\frac{2n}{3}-\frac{i}{3}+\frac{1}{3}-n+j\ge -\frac{n}{3}+\frac{1}{3}+\frac{2j}{3}\ge \frac{1}{3}.
	\end{equation*}
	(b) If $j< \frac{n}{2}$, then $\mathrm{E}(\bar{f}(x^j))=(2n-j)/3$ and $\mathrm{E}(g)= (j-i+1)/3\ge 1/3$. Thus, we have $\mu\ge 1/3$. Then we have 
	\begin{equation*}
	\mathrm{P}(g\le 0)\le \mathrm{P}(|g-\mu|\ge \mu)\le 2e^{-2m\mu^2/(2n)^2}\le 2e^{-m/(18n^2)}=2e^{-n/18},
	\end{equation*} 
	where the second inequality holds by $|f^{\mathrm{n}}(x^{i-1})-f^{\mathrm{n}}(x^j)|\le 2n$ and
	Hoeffding's inequality. \\
	Similar to the discussion at the end of Theorem~\ref{theo-median-segment}, the ERT is polynomial.
\qed $\blacksquare$
\end{proof}
From the proof, we can derive an intuitively explanation for the effectiveness of mean sampling. For $x$ and $z$ satisfying $f(x)>f(z)$ (i.e., $z$ is worse), the expectation of $f^\mathrm{n}(x)$ is larger than $f^\mathrm{n}(z)$. Then, there is a small enough probability to accept $z$ if using mean sampling.  Thus, the search direction of (1+1)-EA will not be misled and the optimal solution can be quickly found.

To prove Theorem~\ref{theo-median-part}, we use Lemma~\ref{lemma-lower}~\cite{18}, which intuitively means that if a true worse solution (i.e., a solution with more 0s) is estimated better than a true better solution with some probability, then we can derive the lower bound for the runtime.
\begin{lemma}[\!\!\cite{18}]\label{lemma-lower}
	If there exists a real number $l\le n/4$ satisfying
	\begin{align}
	&\forall 0<i\le l:\mathrm{P}(F(x^i)<F(x^{i-1})) \le 1-\frac{16i}{n},
	\end{align}
	then  w.h.p., the FHT of (1+1)-EA solving noisy OM is $2^{\Omega(l)}$.
\end{lemma}
\begin{theorem}\label{theo-median-part}
For OM under partial noise, the ERT of (1+1)-EA employing median sampling is exponential.
\end{theorem}
\begin{proof}
	We use Lemma~\ref{lemma-lower} for the proof. Given $0<i<\frac{n}{2}$, let $g=\hat{f}(x^i)-\hat{f}(x^{i-1})$. 
	First we show that $\mathrm{P}(\hat{f}(x^{i-1})=(i-1)/2)\ge 1/3$ for $i< \frac{n}{2}$. Suppose $s$ denotes the number of noisy evaluations satisfying $f^{\mathrm{n}}(x^{i-1})=(i-1)/2 $ in $m$ independent noisy evaluations. We classify $m$ into 2 cases.\\
	(1) $m$ is even. Let \begin{align*}
	 A=\sum_{j=m/2+1}^{m}\binom{m}{j}\left(\frac{2}{3}\right)^j \left(\frac{1}{3}\right)^{m-j},
	B= \sum_{j=0}^{m/2-1}\binom{m}{j}\left(\frac{2}{3}\right)^j \left(\frac{1}{3}\right)^{m-j},
	C=\binom{m}{m/2}\left(\frac{2}{3}\right)^{m/2} \left(\frac{1}{3}\right)^{m/2}.
	\end{align*}
	Note that sum of the three items is 1, and $A\ge B$,
	\begin{equation*}
	A\ge \binom{m}{m/2+1}\left(\frac{2}{3}\right)^{m/2+1} \left(\frac{1}{3}\right)^{m/2-1}
	\ge C.
	\end{equation*}
	Thus, $\mathrm{P}(s\ge m/2+1)=A\ge 1/3$. By definition of median sampling, we derive $\mathrm{P}(\hat{f}(x^{i-1})=(i-1)/2)\ge 1/3$.\\
	(2) $m$ is odd. We have 
	\begin{equation*}
	\mathrm{P}(s\ge (m+1)/2)=\sum_{j=(m+1)/2}^{m}\binom{m}{j}\left(\frac{2}{3}\right)^j \left(\frac{1}{3}\right)^{m-j}\ge \sum_{j=0}^{(m-1)/2}\binom{m}{j}\left(\frac{2}{3}\right)^j \left(\frac{1}{3}\right)^{m-j}.
	\end{equation*}
	Thus, $\mathrm{P}(s\ge (m+1)/2)\ge 1/2$. By the definition of median sampling, we can derive that $\mathrm{P}(\hat{f}(x^{i-1})=(i-1)/2)\ge 1/2$.\\	
	To make $g\ge 0$, it is sufficient that $\hat{f}(x^{i-1})=(i-1)/2$ since it always holds that $\hat{f}(x^{i})\ge i/2$. Thus, $\mathrm{P}(g\ge 0)\ge 1/3$ for $i<\frac{n}{2}$. 
	Then, the condition of Lemma~\ref{lemma-lower} holds by setting $l=n/48$. Thus, the EFHT is $2^{\Omega(l)}=2^{\Omega(n)}$, i.e., exponential.
\qed $\blacksquare$
\end{proof}
From the analysis, we can give an intuitive explanation for the failure of median sampling. Consider $x$ and $z$ satisfying $|x|_0=|z|_0-1$ (that is, $z$ is worse), the 2-quantile of $f^\mathrm{n}(z)$ is larger than that of $f^\mathrm{n}(x)$, and $z$ will be estimated better than $x$ by median sampling w.h.p., implying a wrong comparison.

\section{Application Illustration}
In this section, we provide some guidance for employing median sampling in practice. The theoretical results have revealed that if the 2-quantile of the noisy fitness increases with the true fitness, we can use median sampling to tackle noise. Inspired by this finding, we may use the following three steps to check the effectiveness of median sampling in practice.
\begin{enumerate}
	\item  Find a sequence of solutions with increasing true objective values. Note that the solution space can be very large, and we only need to find some representative solutions. 
	The true fitness of a solution can be obtained by conducting evaluation accurately, instead of using an approximation. For example, a prediction model in machine learning can be evaluated using a large amount of data, and a structure in aerodynamic design can be evaluated by CFDs simulation. Note that the number of representative solutions is very limited, and the evaluation process can be easily parallelized, thus the computational cost is usually acceptable.  
	\item Find an appropriate sample size $m$, such that the 2-quantile of the noisy fitness increases with the sequence. If such sample size doesn't exist or the sample size is too large, it would be better to choose other strategies.
	\item If finding such a sample size, evaluate each solution $m$ times independently and output the median of the $m$ objective values as the estimated fitness during the optimization procedure.
\end{enumerate}
As an application illustration, we use (1+1)-EA to solve OM under onebit noise. It has been known that the ERT of (1+1)-EA solving OM under onebit noise is super-polynomial if the noise probability $p=\omega(\log n/n)$, thus we set $p=\log^2n/n$. We set the problem size $n=100$ and use $0^n,10^{n-1},\ldots,1^n$ as the sequence of solutions with increasing true fitness.
We select the sample size $m$ from $5,10,15,\ldots$, such that the 2-quantile of the noisy fitness increases with the sequence. Figure~\ref{median-seq-onebit} shows that it holds when $m=15$. Thus, using a sample size $m=15$ is probably enough to reduce the negative effect of noise.

To show the effectiveness of median sampling, we next compare the ERT of (1+1)-EA with and without median sampling for the problem size $n \in \{5,10, ... 100\}$. For each $n$, we run (1+1)-EA 100 times independently. In each run, we record the number of fitness evaluations until an optimal solution with respect to the true fitness function is found for the first time. The total number of evaluations of the 100 runs are averaged as the estimation of the ERT. The results are shown in Figure~\ref{median-ert-onebit}. It can be observed that though using median sampling needs to evaluate a solution $m$ times for estimating the fitness, the total number of evaluations required by (1+1)-EA to find the optimum is decreased drastically. 
\begin{figure}[!t]
	\centering
	\begin{minipage}[c]{0.48\textwidth}
		\centering
		\includegraphics[width=0.88\textwidth]{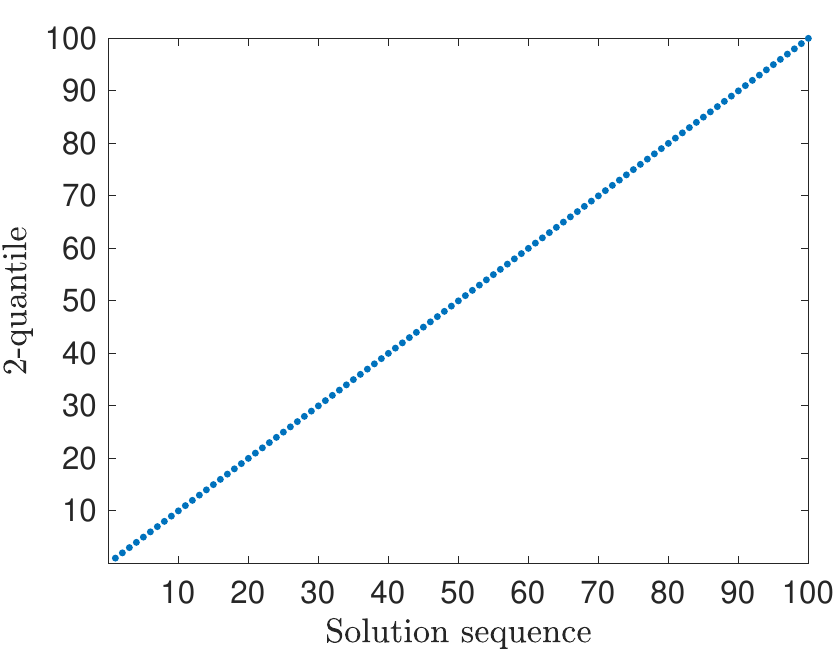}
	\end{minipage}
	\hspace{0.02\textwidth}
	\begin{minipage}[c]{0.48\textwidth}
		\centering
		\includegraphics[width=0.9\textwidth]{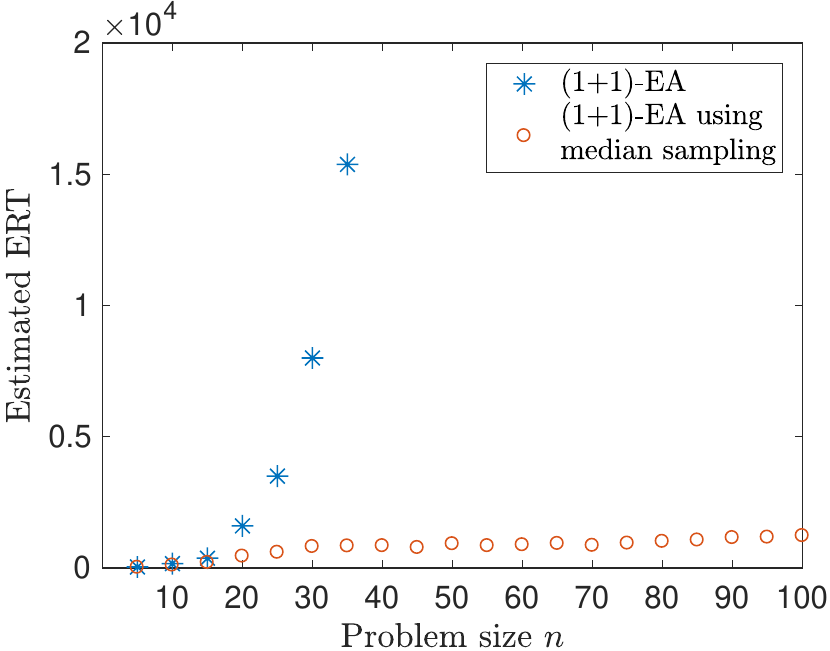}
	\end{minipage}\\[3mm]
	\begin{minipage}[t]{0.48\textwidth}
		\centering
		\caption{2-quantile of the noisy fitness of the sequence $0^n,10^{n-1},\ldots,1^n$. Note that $i$ in $x$-axis corresponds to the $(i+1)$-th solution $1^{i}0^{n-i}$ in the sequence.}
		\label{median-seq-onebit}
	\end{minipage}
	\hspace{0.02\textwidth}
	\begin{minipage}[t]{0.48\textwidth}
		\centering
		\caption{Estimated ERT for (1+1)-EA with and without median sampling on OneMax under onebit noise.}
		\label{median-ert-onebit}
	\end{minipage}
\end{figure}

\section{Conclusion}
In this paper, we introduce median sampling into EAs to handle noise and theoretically analyze the effectiveness of median sampling. We first consider  one classical case, i.e., OM under onebit noise, and show that median sampling can reduce the ERT of (1+1)-EA from exponential to polynomial. Next, by two illustrative examples, we show that when the 2-quantile of the noisy fitness increases with the true objective value, median sampling is better than the commonly used mean sampling; otherwise, it is worse. The results provide us with some guidance to employ median sampling in practice. In the future, it would be interesting to analyze the effect of median sampling on real-world noisy optimization problems.

\Acknowledgements{The authors want to thank the editor and anonymous reviewers for their helpful comments and suggestions, and one reviewer of our work~\cite{25}, whose comments motivate this work. This work was supported by the National Key Research and Development Program of China (2017YFB1003102), the NSFC (62022039, 61672478, 61876077), and the MOE University Scientific-Technological Innovation Plan Program.}

%%%%%%%%%%%%%%%%%%%%%%%%%%%%%%%%%%%%%%%%%%%%%%%%%%%%%%%
%%% Reference section
%%% citation in the content using "some words~\cite{1,2}".
%%% ~ is needed to make the reference number is on the same line with the word before it.
%%%%%%%%%%%%%%%%%%%%%%%%%%%%%%%%%%%%%%%%%%%%%%%%%%%%%%%

%%%%%%%%%%%%%%%%%%%%%%%%%%%%%%%%%%%%%%%%%%%%%%%%%%%%%%%
%%% Appendix sections
%%%%%%%%%%%%%%%%%%%%%%%%%%%%%%%%%%%%%%%%%%%%%%%%%%%%%%%
%\begin{appendix}
%\section{Name}

%\end{appendix}

\end{document}